\documentclass[final]{neus2025}

\usepackage[utf8]{inputenc} 
\usepackage[T1]{fontenc}    
\usepackage{hyperref}       
\usepackage{url}            
\usepackage{booktabs}       
\usepackage{amsfonts}       
\usepackage{nicefrac}       
\usepackage{microtype}      
\usepackage{xcolor}         
\usepackage{cleveref}

\usepackage{thmtools}
\usepackage{thm-restate}

\usepackage{graphicx}
\usepackage{subcaption}

\usepackage{amsmath}    
\usepackage{amssymb}    
\usepackage{bm}         
\usepackage{mathtools}  

\newtheorem{problem}{Problem}
\newtheorem{assumption}{Assumption}
\newtheorem{my_definition}{Definition}

\title[Provably Correct Automata Embeddings]{Provably Correct Automata Embeddings for\\Optimal Automata-Conditioned Reinforcement Learning}
\usepackage{times}

\author{%
 \Name{Beyazit Yalcinkaya} \Email{beyazit@berkeley.edu}\\
 \addr University of California, Berkeley
 \AND
 \Name{Niklas Lauffer} \Email{nlauffer@berkeley.edu}\\
 \addr University of California, Berkeley
 \AND
 \Name{Marcell Vazquez-Chanlatte} \Email{marcell.chanlatte@nissan-usa.com}\\
 \addr Nissan Advanced Technology Center Silicon Valley
 \AND
 \Name{Sanjit A. Seshia} \Email{sseshia@berkeley.edu}\\
 \addr University of California, Berkeley%
}

\begin{document}

\maketitle

\begin{abstract}
    Automata-conditioned reinforcement learning (RL) has given promising results for learning multi-task policies capable of performing temporally extended objectives given at runtime, done by pretraining and freezing automata embeddings prior to training the downstream policy.
    However, no theoretical guarantees were given.
    This work provides a theoretical framework for the automata-conditioned RL problem and shows that it is probably approximately correct learnable.
    We then present a technique for learning provably correct automata embeddings, guaranteeing optimal multi-task policy learning.
    Our experimental evaluation confirms these theoretical results.\footnote{For more information about the project including the code and extensions visit:
\href{https://rad-embeddings.github.io/}{https://rad-embeddings.github.io/}.}
\end{abstract}


\section{Introduction}

Goal-Conditioned Reinforcement Learning (GCRL)  is a framework for learning policies capable of performing multiple tasks given at runtime.
The recent success of foundation models has popularized both natural language (\cite{rocamonde2023vision,brohan2022rt,black2024pi_0}) and demonstrations (\cite{ren2025motion,sontakke2023roboclip}) as ergonomic means of task specification.
Yet, the inherent ambiguity of these instruction modalities remains a challenge for correctness~guarantees.

Formal specifications have been proposed for specifying tasks to goal-conditioned policies.
While their well-defined semantics make them appealing, approaches that rely on hierarchical planning, i.e., planning over the induced automaton of a formal specification and instructing a goal-conditioned policy to execute the plan (\cite{jothimurugan2021compositional,qiu2023instructing}), are inherently suboptimal due to the myopia of their goal-conditioned policies.
Conditioning Reinforcement Learning (RL) policies on Linear Temporal Logic (LTL) specifications was proposed by \cite{vaezipoor2021ltl2action}, using a Graph Neural Network (GNN) to encode abstract syntax trees of LTL formulas.
However, generalization is an inherent limitation due to their use of LTL (\cite{yalcinkaya24compositional}).

In our previous work (\cite{yalcinkaya2023automata,yalcinkaya24compositional}), we proposed using Deterministic Finite Automaton (DFA) as a means of task specification and conditioning the policy on pretrained DFA \emph{embeddings}.
To do so, we first identified a large class of DFAs that capture most of the finite temporal tasks studied in the literature.
We then pretrained a Graph Attention Network (GATv2) (\cite{brody2021attentive}) to map these DFAs to latent vector representations.
Our empirical evaluation demonstrated that conditioning on DFA embeddings enables optimal multi-task policy learning for a large class of DFAs.
However, no theoretical analysis or guarantees were given for this technique.

In this work, we present a theoretical framework for the DFA-conditioned RL problem and show that it is Probably Approximately Correct (PAC)-learnable.
As we showed in \cite{yalcinkaya24compositional}, learning to encode the DFAs while simultaneously learning a policy is challenging due to the sparse reward specified by DFA acceptance.
To this end, in the same work, we demonstrated that pretraining and freezing DFA embeddings and then passing these embeddings to a downstream policy greatly improve learning efficiency.
Therefore, in order for our PAC learnability guarantee to be useful in practice, one must show that such guarantees hold w.r.t. such pretrained and frozen DFA embeddings.
To address this, we present a novel method for learning provably correct DFA embeddings, guaranteeing optimal DFA-conditioned RL.
We first observe that \emph{bisimilar} DFAs represent the same task and then use \emph{bisimulation metrics}, a relaxation of the notion of a bisimulation relation, to embed unique tasks to unique latent representations.
Our experimental evaluation shows that the correctness of the learned DFA embeddings improves downstream policy learning.

\paragraph{Contributions.}
We provide a theoretical framework for DFA-conditioned RL in \Cref{sec:dfa-rl} and prove that it is PAC-learnable in \Cref{thm:pac}.
We present a technique for learning provably correct DFA embeddings in \Cref{sec:bisim}.
Lastly, an empirical evaluation of this approach is given in \Cref{sec:exp}.

\paragraph{Related Work.}
PAC-learnability of RL objectives given by formal specifications has been studied before.
\cite{yang2021tractability} proved that the optimal policy for any LTL formula is PAC-learnable~if and only if the formula can be checked within a finite horizon.
A similar result by \cite{alur2022framework} shows that without additional information on the transition probabilities, such as the minimum nonzero transition probability, LTL is not PAC-learnable.
Later, a positive result for discounted LTL was given by \cite{alur2023policy}.
Our PAC-learnability result can be considered a multi-task generalization of these previous results, where the policy must satisfy a class of specifications, not a single objective.
See \cite{yalcinkaya24compositional} for a more detailed literature review on using formal specification in RL.

\cite{ferns2004metrics} showed that bisimulation metrics can be computed as unique fixed points of a contraction map.
A special case of this result for deterministic dynamics and on-policy samples was proved by \cite{castro2020scalable}.
Later, \cite{zhang2020learning} used bisimulation metrics to learn invariant observation embeddings for RL, where they proved these metrics can be computed while learning a policy.
We will use these results in \Cref{sec:bisim} to learn provably correct DFA embeddings.
To our knowledge, no prior work considered using bisimulation metrics for learning task representations.


\section{Background}

\textbf{Notation.}
Given a set $X$, we write $\mathcal{X}$ to denote a distribution over it, i.e., $\mathcal{X} \in \Delta(X)$, where $\Delta(X) \subset X \to [0, 1]$ represents the set of all distributions over $X$.
We use $\mathbb{I}\{\}$ for the event indicator function, where $\mathbb{I}\{p\} = 1$ if and only if $p$ is true, and $\mathbb{I}\{p\} = 0$ otherwise.


\subsection{Markov Decision Processes}

We model the environment with a Markov Decision Process (MDP), formally defined as follows.

\begin{my_definition}[Markov Decision Process]
A Markov Decision Process (MDP) is defined as the tuple $\mathcal{M} = \langle S, A, P, R, \iota, \gamma \rangle$, where $S$ is the state space, $A$ is the action space, $P: S \times A \to \Delta(S)$ is the transition probability function, $R: S \times A \to \mathbb{R}$ is the reward function, $\iota \in \Delta(S)$ is the initial state distribution, and $\gamma \in [0,1)$ is the discount factor.
An MDP $\mathcal{M}$ is called deterministic if it has a deterministic transition function $T: S \times A \to S$ instead of a probabilistic transition function $P$.
\end{my_definition}

\subsection{Deterministic Finite Automata}

We use Deterministic Finite Automata (DFAs) with three-valued semantics to represent tasks.

\begin{my_definition}[Deterministic Finite Automaton]
    A Deterministic Finite Automaton (DFA) is defined as the tuple $\mathcal{A} = \langle Q, \Sigma, \delta, q_0, F \rangle$, where $Q$ is the finite set of states, $\Sigma$ is the finite alphabet, $\delta: Q \times \Sigma \to Q$ is the transition function, where $\delta(q, a) = q'$ denotes a transition to a state $q' \in Q$ from a state $q \in Q$ by observing a symbol $a \in \Sigma$, $q_0 \in Q$ is the initial state, and $F \subseteq Q$ is the set of final states.
    The three-valued semantics of a DFA is defined by a partition of its final states $F = F^{\top} \cup F^{\bot}$ and its extended (lifted) transition function $\delta^*: Q \times \Sigma^* \to Q$,~where
    \begin{align*}
        \delta^*(q, \varepsilon) & \triangleq q,\\
        \delta^*(q, aw) & \triangleq \delta^*(\delta(q, a), w).
    \end{align*}
    If $\delta^*(q_0, w) \in F^{\top}$, then we say that $\mathcal{A}$ \textbf{accepts} $w$.
    Similarly, $\delta^*(q_0, w) \in F^{\bot}$, then we say that $\mathcal{A}$ \textbf{rejects} $w$.
    $\mathcal{A}$ is called a \textbf{plan DFA} if its final states are sink states, i.e., $\forall q \in F , \forall a \in \Sigma , \delta(q, a) = q$.
\end{my_definition}


DFAs can be minimized to a canonical form (up to a state isomorphism) using the algorithm of \cite{hopcroft1971n}, denoted by $\operatorname{minimize}(\mathcal{A})$.
All minimized plan DFAs with nonempty accepted or rejected languages have single accepting or rejecting states, denoted by $\top$ or $\bot$, respectively, as their final states are all sink states. We denote the single-state accepting DFA by $\mathcal{A}_\top$ and the rejecting one by $\mathcal{A}_\bot$.
Note that the approach presented in this paper is agnostic to the syntax of DFAs and can be trivially extended to other syntactic forms, e.g., \emph{compositional DFAs} from \cite{yalcinkaya24compositional}.

\begin{assumption}\label{assume:mini-dfa}
    In what follows, unless stated otherwise, all DFAs are plan DFAs.
\end{assumption}

We use DFAs to represent temporal tasks, which can be understood as \emph{plans}.
However, one can define multiple DFAs for the same task, i.e., DFAs without any additional assumptions are not canonical task representations.
Therefore, we will need a notion of similarity between DFAs so that we can compare the tasks represented by them--we define \emph{bisimulation relation over DFAs} next.

\begin{my_definition}[Bisimulation Relation over DFAs]\label{defn:dfa-bisim}
Given two DFAs
$\mathcal{A} = \langle Q, \Sigma, \delta, q_0, F \rangle$ and $\mathcal{A}' = \langle Q', \Sigma, \delta', q_0', F' \rangle$ over the same alphabet \(\Sigma\). A relation $B \subseteq Q \times Q'$ is called a \textbf{bisimulation relation} between \(\mathcal{A}\) and \(\mathcal{A}'\) if the following conditions hold:
\begin{enumerate}
    \item \((q_0, q_0') \in B\).
    \item For all \((q, q') \in B\), $q \in F^{\top} \iff q' \in F'^{\top}$ and $q \in F^{\bot} \iff q' \in F'^{\bot}$.
    \item For all \((q, q') \in B\) and \(a \in \Sigma\), $(\delta(q, a), \delta'(q', a)) \in B$.
\end{enumerate}
We say that \(\mathcal{A}\) and \(\mathcal{A}'\) are \textbf{bisimilar}, denoted by $\mathcal{A} \sim \mathcal{A}'$, if there exists such a bisimulation relation.
\end{my_definition}

A bisimulation relation over DFAs is an equivalence relation on the DFA states preserving both the transition structure and the acceptance condition--meaning if two states are related under this relation, then for every input symbol, their successor states are also related, and they either both accept or both reject. Bisimilar DFAs are behaviorally indistinguishable--they represent the same~task.

\section{DFA-Conditioned Reinforcement Learning}\label{sec:dfa-rl}

We will define the DFA-conditioned RL problem over a distribution of DFAs, similar to the GCRL problem given in \Cref{sec:gcrl}.
However, we need some extra structure over these DFAs since, in our case, goals are not simply sets of states but are DFAs encoding temporal tasks given to the policy.

\begin{my_definition}[DFA Space]\label{defn:dfa-space}
    A set of DFAs $D_{\Sigma, n}$ over some shared alphabet $\Sigma$ and with at most $n$ states is called a \textbf{DFA space} if $\mathcal{A}_\top, \mathcal{A}_\bot \in D_{\Sigma, n}$ and taking any path in a DFA from $D_{\Sigma, n}$ and minimizing the resulting DFA gives a DFA in $D_{\Sigma, n}$, i.e.,
    \[
        \forall \mathcal{A} = \langle Q, \Sigma, \delta, q_0, F \rangle \in D_{\Sigma, n} , \forall w \in \Sigma^*, \quad \operatorname{minimize}(\mathcal{A}' = \langle Q, \Sigma, \delta, \delta^*(q_0, w), F \rangle) \in D_{\Sigma, n}.
    \]
    A DFA space $D_{\Sigma, n}$ induces an MDP defined by the tuple
$\mathcal{M}_{D_{\Sigma, n}} = \langle D_{\Sigma, n}, \Sigma, T_{D_{\Sigma, n}}, R_{D_{\Sigma, n}} \rangle$,
where
\begin{itemize}
    \item $D_{\Sigma, n}$, the set of DFAs, is the set of states,
    \item $\Sigma$, the shared alphabet, is the set of actions,
    \item $T_{D_{\Sigma, n}}: D_{\Sigma, n} \times \Sigma \to D_{\Sigma, n}$ is the transition function defined by
    \[T_{D_{\Sigma, n}}(\mathcal{A} = \langle Q, \Sigma, \delta, q_0, F \rangle, a) = \operatorname{minimize}(\mathcal{A}' = \langle Q, \Sigma, \delta, \delta(q_0, a), F \rangle)\text{, and}\]
    \item $R_{D_{\Sigma, n}} : D_{\Sigma, n} \times \Sigma \to \{-1, 0, 1\}$ is the reward function defined by
    \[
        R_{D_{\Sigma, n}}(\mathcal{A}_t, a) = \begin{cases}
            1 &\text{if } T_{D_{\Sigma, n}}(\mathcal{A}_t, a) = \mathcal{A}_\top\\
            -1 &\text{if } T_{D_{\Sigma, n}}(\mathcal{A}_t, a) = \mathcal{A}_\bot\\
            0 &\text{otherwise.}
        \end{cases}
    \]
\end{itemize}
\end{my_definition}

A DFA space is a set of DFAs closed under random walks, i.e., taking any random path in a DFA from this set and pruning the unreachable states results in a DFA in the set.
In other words, one cannot get a DFA outside this set by taking a random walk with minimization, hence the name \emph{space}.

\begin{assumption}
    $D_{\Sigma, n}$ denotes a DFA space, and $\mathcal{D}_{\Sigma, n} \in \Delta(D_{\Sigma, n})$ is a distribution over it.
\end{assumption}

We now have all the theoretical machinery needed to formally state the DFA-conditioned RL problem.
We first define the environment model and then continue with the statement of the problem.

\begin{my_definition}[DFA-Conditioned MDP]\label{defn:dfa-mdp}
    Let $\mathcal{M} = \langle S, A, P, R, \iota, \gamma \rangle$ be an MDP, $D_{\Sigma, n}$ be a DFA space, and $L: S \to \Sigma$ be a labeling function.
    A DFA-conditioned MDP is the cascade composition of $\mathcal{M}$ and $\mathcal{M}_{D_{\Sigma, n}}$, using $L$ to map states to alphabet symbols, defined by
    \[
    \mathcal{M} \mid_L \mathcal{M}_{D_{\Sigma, n}} = 
    \langle S \times D_{\Sigma, n}, A, P_{\mathcal{M} \mid_L \mathcal{M}_{D_{\Sigma, n}}}, R_{\mathcal{M} \mid_L \mathcal{M}_{D_{\Sigma, n}}}, \iota_{\mathcal{M} \mid_L \mathcal{M}_{D_{\Sigma, n}}}, \gamma \rangle
    \]
    where:
    \begin{itemize}
        \item $S \times D_{\Sigma, n}$ is the (product) state space,
        \item $A$ is the action space (of $\mathcal{M}$),
        \item $P_{\mathcal{M} \mid_L \mathcal{M}_{D_{\Sigma, n}}} : S \times D_{\Sigma, n} \times A \to \Delta(S \times D_{\Sigma, n})$ is the transition probability function defined~by
        \[
        P_{\mathcal{M} \mid_L \mathcal{M}_{D_{\Sigma, n}}}(s,\mathcal{A}, a, s',\mathcal{A}') = P(s,a,s') \mathbb{I} \left\{ \mathcal{A}' = T_{D_{\Sigma, n}}\left(\mathcal{A}, L(s')\right) \right\},
        \]
        \item $R_{\mathcal{M} \mid_L \mathcal{M}_{D_{\Sigma, n}}} : S \times D_{\Sigma, n} \times A \to \{-1, 0, 1\}$ is the reward function defined by
        \[
        R_{\mathcal{M} \mid_L \mathcal{M}_{D_{\Sigma, n}}}(s, \mathcal{A}, a) = \begin{cases}
        1 & \text{if } T_{D_{\Sigma, n}}(\mathcal{A}, L(s')) = \mathcal{A}_\top\\
        -1 & \text{if } T_{D_{\Sigma, n}}(\mathcal{A}, L(s')) = \mathcal{A}_\bot\\
        0 & \text{otherwise,}
    \end{cases}
        \]
        where $s' \sim P(s, a)$ is the next MDP state.
        \item $\iota_{\mathcal{M} \mid_L \mathcal{M}_{D_{\Sigma, n}}} \in \Delta( S \times D_{\Sigma, n} )$ is the initial state distribution defined by
        \[\iota_{\mathcal{M} \mid_L \mathcal{M}_{D_{\Sigma, n}}}(s, \mathcal{A}) = \iota(s) \mathcal{D}_{\Sigma, n}(\mathcal{A}),\]
        and 
        \item $\gamma \in [0,1)$ is the discount factor (of $\mathcal{M}$).
    \end{itemize}
    
\end{my_definition}

A DFA-conditioned MDP essentially couples an MDP with a DFA space, where the policy interacts with the MDP while simultaneously navigating the DFA space to reach the accepting DFA.
Next, we formalize this notion and finally state the DFA-conditioned RL problem.

\begin{my_definition}[DFA-Conditioned Reinforcement Learning Problem]\label{defn:dfa-rl}
Given an DFA-conditioned MDP $\mathcal{M} \mid_L \mathcal{M}_{D_{\Sigma, n}}$ as defined in \Cref{defn:dfa-mdp}, a \textbf{DFA-conditioned policy} is a mapping
\[
\pi: S \times D_{\Sigma, n} \to \Delta(A),
\]
that assigns to each pair \((s,\mathcal{A})\) a probability distribution over the action space \(A\).
The \textbf{DFA-conditioned RL problem} is to find a policy \(\pi\) maximizing expected cumulative discounted reward:
\[
J_{\mathcal{M} \mid_L \mathcal{M}_{D_{\Sigma, n}}}(\pi) = \mathbb{E}_{(s_0,\mathcal{A}_0) \sim \iota_{\mathcal{M} \mid_L \mathcal{M}_{D_{\Sigma, n}}}}\left[ \sum_{t=0}^{\mathcal{A}_t = \mathcal{A}_\top \lor \mathcal{A}_t = \mathcal{A}_\bot} \gamma^t \, R_{\mathcal{M} \mid_L \mathcal{M}_{D_{\Sigma, n}}}(s_t, \mathcal{A}_t, a_t) \right],
\]
where trace \(\{(s_t,\mathcal{A}_t, a_t)\}_{t\ge0}\) is generated by:
\[
    a_{t} \sim \pi(s_t, \mathcal{A}_t), \quad s_{t+1} \sim P(s_t,a_t), \quad \mathcal{A}_{t+1} = T_{D_{\Sigma, n}}\left(\mathcal{A}_t,\, L(s_{t+1})\right),
\]
until the accepting or rejecting DFA is reached, i.e., $\mathcal{A}_t = \mathcal{A}_\top \lor \mathcal{A}_t = \mathcal{A}_\bot$.
The objective is to solve
\[
\pi^* = \arg\max_{\pi} J_{\mathcal{M} \mid_L \mathcal{M}_{D_{\Sigma, n}}}(\pi),
\]
i.e., to learn a policy that maximizes the probability of satisfying a given temporal specification from $D_{\Sigma, n}$ (by driving its DFA representation to \(\mathcal{A}_\top\)) while operating in the underlying MDP $\mathcal{M}$.
\end{my_definition}

Notice the difference between the GCRL problem formulation given in \Cref{sec:gcrl} and the DFA-conditioned one given above.
Specifically, in GCRL, goals are static, i.e., they are given at the beginning, and the policy conditions on the same goal until it is accomplished. In the DFA-conditioned RL setting, through the labeling function, a given DFA task also evolves (based on the transition dynamics of its DFA space) as the policy interacts with the underlying MDP.
Therefore, the policy is essentially learning to navigate two MDPs: the underlying MDP and the DFA space.

\subsection{PAC-Learnability of the DFA-Conditioned Reinforcement Learning Problem}

We show that the DFA-conditioned RL problem is Probably Approximately Correct (PAC)-learnable.
To do so, we will use the PAC-MDP framework introduced by~\cite{strehl2006pac}.
An RL algorithm is called PAC-MDP (PAC in MDPs) if it finds a near-optimal policy with high probability in any MDP after a number of interactions that is polynomial in the problem’s key parameters, stated next.

\begin{my_definition}[Probably Approximately Correct Learnability in  MDPs]\label{defn:pac}
    A learning algorithm \( \mathbb{A} \) is said to be \textbf{Probably Approximately Correct in MDPs (PAC-MDP)} if for any MDP \( \mathcal{M} = (S, A, T, R, \iota, \gamma) \), \( \epsilon > 0 \), and \( p \in (0,1) \), there exists a polynomial function
    \[
    N = f\left(|S|, |A|, \frac{1}{\epsilon}, \frac{1}{p}, \frac{1}{1-\gamma}\right)
    \]
    s.t., with probability at least \( 1 - p \), the total number of time steps during which the policy \( \pi \) (current policy being trained) executed by \( \mathbb{A} \) is more than \( \epsilon \)-suboptimal is at most $N$, i.e., we have
    \[
        \left|\left\{ t \ge 0 \;:\; V^{\pi}(s_t) < V^*(s_t) - \epsilon \right\}\right| \leq N
    \]
    with probability at least \( 1-p \),
    where $V^{\pi}$ denotes the current value function and $V^{*}$ is the optimal~one.
\end{my_definition}

We want to show that if an algorithm is PAC-MDP, then it is also PAC in any DFA-conditioned MDP.
Observe that, in \Cref{defn:dfa-mdp}, we take the cascade composition of the underlying MDP and the MDP induced by the DFA space which is finite, giving us the following PAC-learnability result.

\begin{restatable}{thm}{pac}\label{thm:pac}
    If a learning algorithm \( \mathbb{A} \) is PAC-MDP as defined in \Cref{defn:pac}, then for any DFA-conditioned MDP $\mathcal{M} \mid_L \mathcal{M}_{D_{\Sigma, n}}$, \( \epsilon > 0 \), and \( p \in (0,1) \), there exists a polynomial function
    \[
    N' = f\left(|S|\cdot|D_{\Sigma,n}|, |A|, \frac{1}{\epsilon}, \frac{1}{p}, \frac{1}{1-\gamma}\right)
    \]
    s.t. the total number of \(\epsilon\)-suboptimal steps taken by \( \mathbb{A} \) is at most \(N'\) with probability at least \( 1 - p \).
\end{restatable}

The proof is in the appendix.
\Cref{thm:pac} proves that the DFA-conditioned RL problem is PAC-learnable, assuming the underlying MDP is solvable.
However, in practice, one cannot input a DFA to a policy directly. Instead, one uses an encoder (possibly pretrained) mapping DFAs to embeddings. In such cases, the optimality of the learned DFA-conditioned policy depends on the encoder.

\section{Learning Provably Correct Automata Embeddings}\label{sec:bisim}

In the previous section, we introduced the idealized, theoretical formulation of the DFA-conditioned RL problem and proved that it is PAC-learnable.
However, a policy implemented by a feed-forward neural network, as is usually the case, cannot condition on a DFA directly, but rather an encoding of the DFA is required, such as a vector representation.
Then, the question is whether the policy conditioning on DFA encodings is optimal w.r.t. the theoretical formulation of the problem.
This is the problem we tackle in the following, but before doing so, we formally state the problem.



\begin{problem}\label{problem:main}
    Given a DFA space $D_{\Sigma, n}$, learn an encoder $\phi: D_{\Sigma, n} \to \mathcal{Z}$ s.t. for any MDP $\mathcal{M}$ and labeling function $L$, solving $\mathcal{M} \mid_L \mathcal{M}_{D_{\Sigma, n}}$ with a policy $\pi_\mathcal{Z}: S \times \mathcal{Z} \to \Delta(A)$ conditioning on DFA embeddings is equivalent to solving it with a DFA-conditioned policy $\pi : S \times D_{\Sigma, n} \to \Delta(A)$, i.e.,
    \[
        \forall \mathcal{M} , \forall L, \quad \arg\max_{\pi} J_{\mathcal{M} \mid_L \mathcal{M}_{D_{\Sigma, n}}}(\pi) = \arg\max_{\pi_\mathcal{Z} \circ \phi} J_{\mathcal{M} \mid_L \mathcal{M}_{D_{\Sigma, n}}}(\pi_{\mathcal{Z}} \circ \phi),
    \]
    where $\pi_{\mathcal{Z}} \circ \phi (s, \mathcal{A}) = \pi_\mathcal{Z}(s, \phi(\mathcal{A}))$.
\end{problem}


\begin{assumption}\label{assume:capacity}
    $\phi: D_{\Sigma, n} \to \mathcal{Z}$ has enough capacity to represent DFAs in its domain.
    That is, the learnable encoder $\phi$ has a parametrization that can map distinct DFAs in $D_{\Sigma, n}$ to unique~embeddings.
\end{assumption}

Intuitively, we want a policy conditioning on the latent representations of DFAs (rather than DFAs themselves) to be equivalent to the theoretical formulation given in \Cref{defn:dfa-rl}, i.e., one finds the optimal solution whenever the other does so.
Observe that even under \Cref{assume:capacity}, one does not get such a guarantee directly since the claim is not only an expressivity argument but also involves proving that the training procedure of the encoder provides such representations.
In the following, we present a training technique for such encoders and prove that it solves \Cref{problem:main}.
Our method involves learning a \emph{bisimulation metric} over the induced MDP of a given DFA space.
Therefore, we first define bisimulation metrics and then show how they can be computed in deterministic MDPs.




\subsection{Bisimulation Relations and Metrics over MDP states}

A bisimulation metric can be viewed as a relaxation of the notion of a bisimulation relation over MDP states. We start by defining the latter and then continue with the former.

\begin{my_definition}[Bisimulation Relation over MDP states]\label{defn:bisim-mdp}
Let \(\mathcal{M} = \langle S, A, P, R, \iota, \gamma \rangle\) be an MDP. A relation $B \subseteq S \times S$ is called a \textbf{bisimulation relation} if for every pair \((s,t) \in B\) and for every action \(a \in A\), the following conditions hold:
\begin{enumerate}
    \item $R(s,a) = R(t,a)$.
    \item For all \(B\)-closed set \(X \subseteq S\) (i.e., if \(x \in X\) and \((x,y) \in B\) then \(y \in X\)), 
    \[
    \sum_{x \in X}P(s,a,x) = \sum_{x \in X}P(t,a,x).
    \]
\end{enumerate}
We say \(s,t \in S\) are \textbf{bisimilar}, denoted by \(s \sim_{\mathcal{M}} t\), if there is a bisimulation relation \(B\) s.t.~\((s,t) \in B\).
\end{my_definition}

Intuitively, two states are bisimilar if they are behaviorally indistinguishable--taking any action in either state yields the same immediate reward and leads to similar probabilistic outcomes, so an agent cannot tell them apart when making decisions.
Next, we show how this relates to \Cref{defn:dfa-bisim}.

\begin{restatable}{lmm}{nobisim}\label{lemma:no-bisim}
    Given a DFA space $D_{\Sigma, n}$, two DFAs $\mathcal{A}, \mathcal{A}' \in D_{\Sigma, n}$ are bisimilar if and only if they are bisimilar states in the induced deterministic MDP $\mathcal{M}_{D_{\Sigma, n}}$, i.e.,
    \[
        \forall \mathcal{A}, \mathcal{A}' \in \mathcal{D}_{\Sigma, n} , \quad
        \mathcal{A} \sim \mathcal{A}'
        \iff
        \mathcal{A} \sim_{\mathcal{M}_{D_{\Sigma, n}}} \mathcal{A}'.
    \]
\end{restatable}

The proof is given in the appendix.
Essentially, \Cref{lemma:no-bisim} shows that \Cref{defn:dfa-bisim} and \Cref{defn:bisim-mdp} are equivalent for DFAs in a DFA space, which will later help us with \Cref{problem:main}.
We will continue with the formal definition of a bisimulation metric.
But before doing so, informally, a function \( d: X \times X \to \mathbb{R}_{\geq 0} \) is called a \emph{pseudometric} on a set \( X \) if it satisfies non-negativity, symmetry, and the triangle inequality but may allow \( d(x, y) = 0 \) for \( x \neq y \), see \Cref{sec:metric} for a formal definition.

\begin{my_definition}[Bisimulation Metric]\label{defn:bisim-metric}
Let $\mathcal{M} = \langle S, A, P, R, \iota, \gamma \rangle$ be an MDP.
A pseudometric $d$ is called a \textbf{bisimulation metric} if the equivalence relation induced by $d$ is a bisimulation relation, i.e.,
\[
    \sim_{\mathcal{M}} = \{(s,t) \in S \times S \mid_L d(s, t) = 0\}.
\]
\end{my_definition}

Recall that a given DFA space $D_{\Sigma, n}$ induces a deterministic MDP $\mathcal{M}_{D_{\Sigma, n}}$, where each state of this MDP is a DFA.
We want our learned encoder to uniquely distinguish between different behaviors, but we do not care whether we can distinguish between different representations of the same task.
Therefore, we can use a bisimulation metric to measure \emph{how bisimilar} two DFAs are and utilize this idea to learn a provably correct embedding space by ensuring that if two DFAs are not bisimilar, then they must have different embeddings.
To this end, we first present the following result stating that bisimulation metrics over deterministic MDPs can be computed as fixed-point solutions.

\begin{restatable}{thm}{operator}\label{thm:operator}
    Let $\mathcal{M} = \langle S, A, T, R, \iota, \gamma \rangle$ be a deterministic MDP.
    Define operators:
    \begin{align*}
        d^{k}(s, t) &\gets \left| R(s, \pi^{k}(s, t)) - R(t, \pi^{k}(s, t)) \right| + \gamma\, d^{k - 1} \left( T(s, \pi^{k}(s, t)),\, T(t, \pi^{k}(s, t)) \right), \\
        \pi^k(s, t) &\gets \arg\max_{a \in A} \left\{ \left| R(s,a) - R(t,a) \right| + \gamma\, d^{k - 1} \left( T(s, a),\, T(t, a) \right) \right\}.
    \end{align*}
    Then, there exists unique fixed points $d^*$ and $\pi^*$, and $d^*$ is a bisimulation metric.
\end{restatable}

The proof is given in the appendix. \Cref{thm:operator} is essentially an adaptation of the results previously given in this domain to our setting.
Specifically, \cite{ferns2004metrics} first proved that a bisimulation metric can be computed as a unique fixed point of a contraction map.
\cite{castro2020scalable} later showed special cases of this result for deterministic MDPs and for on-policy variants where actions are given by a policy.
\cite{zhang2020learning} then presented a result showing that one can learn a bisimulation metric jointly while learning a control policy predicting actions for a downstream task.
We combine these results to show that a bisimulation metric can be computed while simultaneously learning a policy maximizing it.
Given $\pi^*$, a bisimulation metric $d^{*}$ can be computed up to an $\alpha$ accuracy by iteratively applying the operator $\left\lceil\frac{\ln \alpha}{\ln \gamma}\right\rceil$ times, with an overall complexity of $O\left( |A| |S|^4 \log |S| \frac{\ln \alpha}{\ln \gamma} \right)$.
%

\subsection{Learning Automata Embeddings by Computing Bisimulation Metrics}


Given a DFA space $D_{\Sigma, n}$, to learn an encoder $\phi : D_{\Sigma, n} \to \mathcal{Z}$ that solves \Cref{problem:main}, we train it to learn latent representations s.t. their normalized $\ell_2$-norms form a bisimulation metric.
To do so, we use the operators from \Cref{thm:operator} and define our pseudometric as follows:
\begin{align}\label{eqn:metric}
    d(\mathcal{A}, \mathcal{A}') \triangleq \| \hat{\phi}(\mathcal{A}) - \hat{\phi}(\mathcal{A}') \|_2,
\end{align}
where $\hat{\phi}(\mathcal{A}) = \frac{\phi(\mathcal{A})}{\| \phi(\mathcal{A}) \|_2}$ denotes vector normalization.
Since it is hard to compute the $\arg\max$ in \Cref{thm:operator}, we simultaneously learn a policy $\pi: \mathcal{Z} \times \mathcal{Z} \to \Delta(\Sigma)$ approximating it in the latent space.
We generate episodes starting from $\mathcal{A}_0, \mathcal{A}_0' \sim \mathcal{D}_{\Sigma, n}$ and evolving as follows:
\[
    a_{t} \sim (\pi \circ \phi)(\mathcal{A}_t, \mathcal{A}_t'),
    \quad
    \mathcal{A}_{t + 1} = T_{D_{\Sigma, n}}(\mathcal{A}_t, a_t),
    \quad
    \mathcal{A}_{t + 1}' = T_{D_{\Sigma, n}}(\mathcal{A}_t', a_t),
\]
where $(\pi \circ \phi)(\mathcal{A}_t, \mathcal{A}_t') = \pi(\phi(\mathcal{A}_t), \phi(\mathcal{A}_t'))$.
We use Proximal Policy Optimization (PPO) by \cite{schulman2017proximal} to jointly learn $\pi$ and $\phi$ with the following objective:
\begin{align}\label{eqn:obj}
    J_{D_{\Sigma, n}}(\pi \circ \phi)
    =
    J_{\text{clip}}(\pi \circ \phi)
    +
    J_{\text{val}}(\phi),
\end{align}
where $J_{\text{clip}}(\pi \circ \phi)$ is the clipped surrogate objective computed using \Cref{eqn:metric} as its value function, i.e., $V_t = d(\mathcal{A}_t, \mathcal{A}_t')$.
The details of $J_{\text{clip}}(\pi \circ \phi)$ and PPO are not relevant to us; however, note that while it is not guaranteed, it usually finds the optimal solution, see \cite{schulman2017proximal} for details.
The second term in the objective given in \Cref{eqn:obj}, the value objective, is defined as follows:
\begin{align*}
    J_{\text{val}}(\phi)
    &=
    -
    \left(
    V_t - \left( \left| R_{D_{\Sigma, n}}(\mathcal{A}_t, a_t) - R_{D_{\Sigma, n}}(\mathcal{A}_t', a_t) \right| + \gamma\bar V_{t + 1}\right)
    \right)^2
    \\
    &=
    -
    \left(
    d(\mathcal{A}_t, \mathcal{A}_t') - \left(\left|R_{D_{\Sigma, n}}(\mathcal{A}_t, a_t) - R_{D_{\Sigma, n}}(\mathcal{A}_t', a_t)\right| + \gamma \bar{d}(\mathcal{A}_{t + 1}, \mathcal{A}_{t + 1}') \right)
    \right)^2,
\end{align*}
where $\bar{V}_{t + 1}$ and $\bar{d}(\mathcal{A}_{t + 1}, \mathcal{A}_{t + 1}')$ denotes calls with stop gradients, i.e., no gradient flow to $\phi$.
$J_{\text{val}}(\phi)$ implements the objective for the pseudometric given in \Cref{thm:operator}, penalizing for diverging from the one-step lookahead target.
The combined objective of $\phi$ is then to learn latent representations that form a bisimulation metric under normalized $\ell_2$-norm while also providing representations for $\pi$.
Next, we show that an encoder maximizing \Cref{eqn:obj} maps two DFAs to the same embedding if and only if they are bisimilar, therefore proving that such encoders can distinguish distinct tasks.

\begin{restatable}{lmm}{unique}\label{thm:unique}
Let \(D_{\Sigma, n}\) be a DFA space, $\phi^*$ be an encoder, and $\pi^*$ be a policy s.t.
\[
    \pi^* \circ \phi^* = \arg\max_{\pi \circ \phi} J_{D_{\Sigma, n}}(\pi \circ \phi),
\]
where $J_{D_{\Sigma, n}}(\pi \circ \phi)$ is given by \Cref{eqn:obj}.
Then, $\phi^*$ satisfies:
\[
\forall \mathcal{A}, \mathcal{A}' \in D_{\Sigma, n} , \quad \mathcal{A} \sim \mathcal{A}' \iff \phi^*(\mathcal{A}) = \phi^*(\mathcal{A}').
\]
\end{restatable}

The proof is given in the appendix.
Observe that if our trained encoder can distinguish between DFAs that are not bisimilar, then it solves \Cref{problem:main}, as bisimilar DFAs are different representations for the same task--no need to distinguish them.
Next, we formally state this result, solving \Cref{problem:main}.

\begin{restatable}{thm}{main}\label{thm:main}
    Let \(D_{\Sigma, n}\) be a DFA space, $\phi$ be an encoder, and $\pi^*$ be a policy s.t.
    \[
        \pi^* \circ \phi^* = \arg\max_{\pi \circ \phi} J_{D_{\Sigma, n}}(\pi \circ \phi),
    \]
    where $J_{D_{\Sigma, n}}(\pi \circ \phi)$ is given by \Cref{eqn:obj}.
    Then, $\phi^*$ solves \Cref{problem:main}.
\end{restatable}

The proof is given in the appendix.
Intuitively, since our encoder can distinguish bisimilar DFAs and bisimilar DFAs represent the same task, one can equivalently reformulate the DFA-conditioned RL problem given in \Cref{defn:dfa-rl}, which is defined over $D_{\Sigma, n}$, as one over $\mathcal{Z}$, solving \Cref{problem:main}.

\section{Experiments}\label{sec:exp}

We implemented the technique given in \Cref{sec:bisim} using a GATv2 model as our DFA encoder and Reach-Avoid Derived (RAD) DFAs with at most $10$ states, which are plan DFAs, both presented in \cite{yalcinkaya24compositional}.
One difference in our GATv2 model is that given a DFA with $n$ states, we do $n$ message-passing steps, instead of doing it for a fixed number as in \cite{yalcinkaya24compositional}. 
To break the symmetry, caused by taking the absolute value of the reward difference, we trained the policy using the reward difference without the absolute value.
\Cref{fig:bisim-curves} shows that our training technique finds the fixed point, where the objectives from \Cref{sec:bisim} are given as losses.
We then tested the accuracy as well as the generalization capabilities of these DFA embeddings.
To do so, we generated RAD, Reach (R), and Reah-Avoid (RA) DFAs.
During training the number of states of a RAD DFA was sampled from a truncated geometric distribution (with $10$ as the upper bound) whereas during testing we sampled it using a bounded uniform distribution.
We also generated out-of-distribution (OOD) DFAs with the number of states sampled uniformly between $11$ and $20$.
We computed bisimulation metrics, i.e., the normalized $\ell_2$-norms, between the embeddings of these DFAs.
\Cref{fig:heatmap} gives these results in the form of a heatmap, demonstrating the correctness of the learned DFA embeddings--$0$ on the diagonal.
We further checked whether any of these sampled DFAs (both in-distribution and OOD ones) are mapped to the same embedding (up to a $10^{-8}$ accuracy) or not, and we confirm that the encoder has a $100\%$ success rate in these samples.
\Cref{fig:policy} compares our new pretraining technique from \Cref{sec:bisim} with our previous pretraining procedure based on solving DFAs (\cite{yalcinkaya24compositional}), which does not guarantee correctness, showing that the correctness of DFA embeddings improves downstream policy learning.
All results are over 5 seeds.

\begin{figure*}[t]
    \centering
    \subfigure[Learning curves] 
    {
        \includegraphics[width=0.3\linewidth]{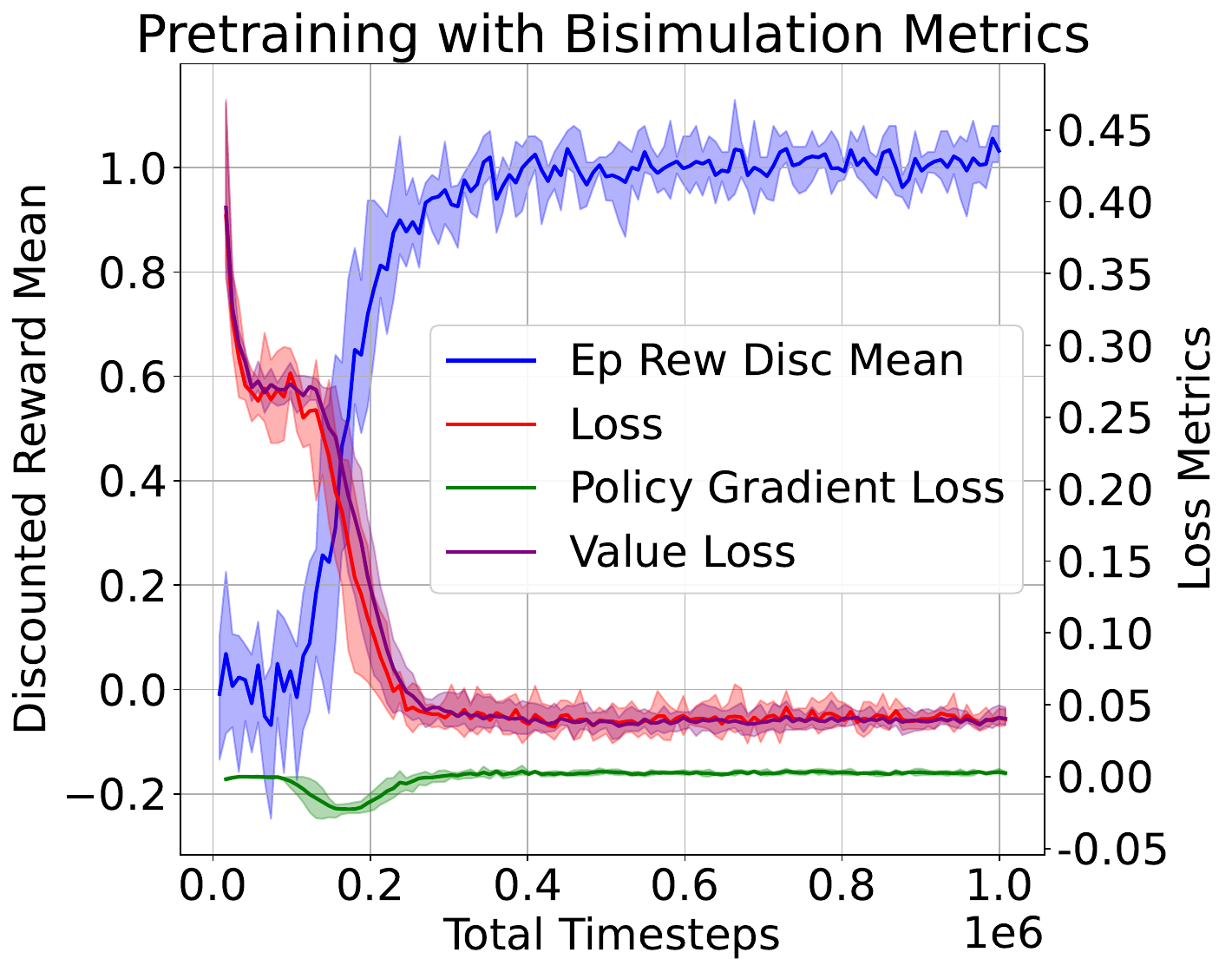}
        \label{fig:bisim-curves}
    }%
    \subfigure[Normalized $\ell_2$ distances] 
    {
        \includegraphics[width=0.3\linewidth]{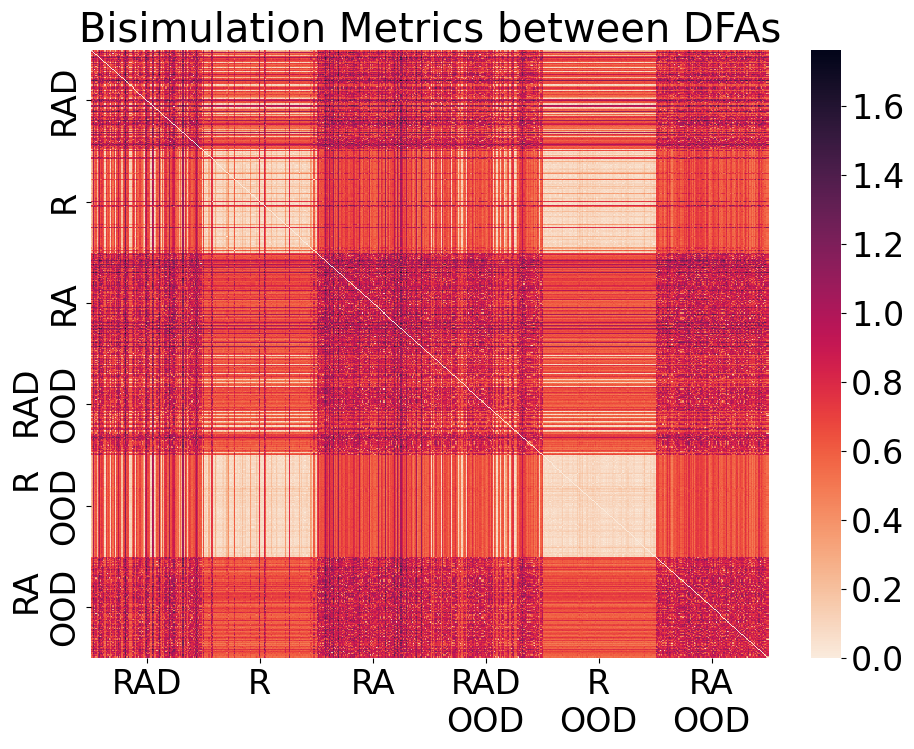}
        \label{fig:heatmap}
    }%
    \subfigure[DFA-condition policies] 
    {
        \includegraphics[width=0.3\linewidth]{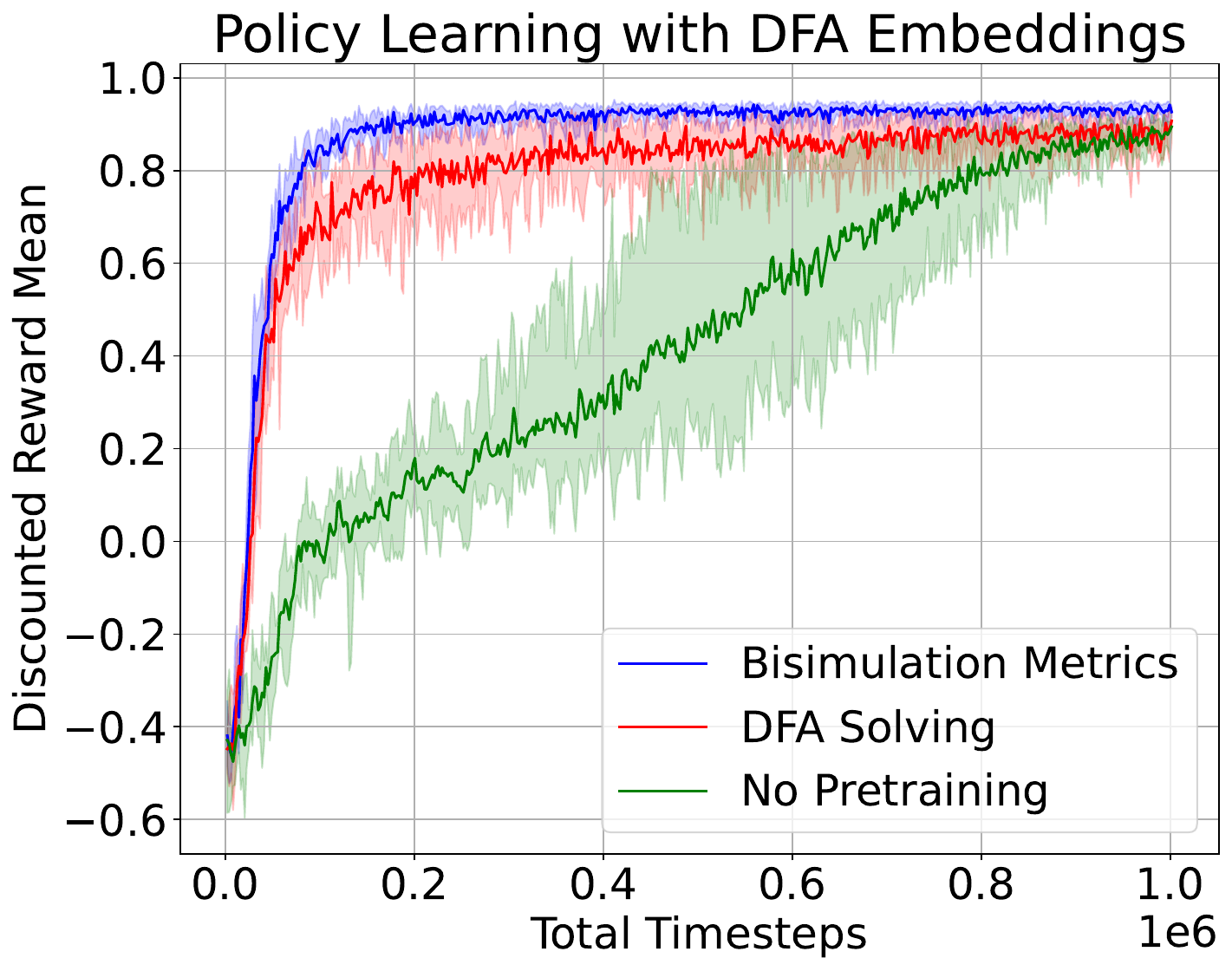}
        \label{fig:policy}
    }
    \caption{Left: learning curves for the method in \Cref{sec:bisim}. Center: heatmap of \Cref{eqn:metric} between various DFAs. Right: learning DFA-conditioned policies for RA tasks with number of states~sampled uniformly from $[3,6]$, comparing DFA embeddings from \Cref{sec:bisim} and \cite{yalcinkaya24compositional}.}
    \label{fig:photos}
\end{figure*}

\section{Conclusion}
In this work, we established a theoretical framework for DFA-conditioned RL and showed its PAC-learnability. We then introduced a method for learning provably correct automata embeddings, ensuring optimal multi-task policy learning. Our approach builds on the promising results of DFA-conditioned RL, leveraging pretrained and frozen DFA embeddings to enable the learning of policies for temporally extended objectives specified at runtime. Our experimental evaluation confirms the theoretical guarantees of our method, demonstrating DFA-embeddings enable optimal multi-task RL.

\acks{This work is partially supported by the DARPA contract FA8750-23-C-0080 (ANSR), by Nissan and Toyota under the iCyPhy Center, and by C3DTI.
Niklas Lauffer is supported by an NSF fellowship.}

\bibliography{references}

\begin{thebibliography}{22}
\providecommand{\natexlab}[1]{#1}
\providecommand{\url}[1]{\texttt{#1}}
\expandafter\ifx\csname urlstyle\endcsname\relax
  \providecommand{\doi}[1]{doi: #1}\else
  \providecommand{\doi}{doi: \begingroup \urlstyle{rm}\Url}\fi

\bibitem[Alur et~al.(2022)Alur, Bansal, Bastani, and Jothimurugan]{alur2022framework}
Rajeev Alur, Suguman Bansal, Osbert Bastani, and Kishor Jothimurugan.
\newblock A framework for transforming specifications in reinforcement learning.
\newblock In \emph{Principles of Systems Design: Essays Dedicated to Thomas A. Henzinger on the Occasion of His 60th Birthday}, pages 604--624. Springer, 2022.

\bibitem[Alur et~al.(2023)Alur, Bastani, Jothimurugan, Perez, Somenzi, and Trivedi]{alur2023policy}
Rajeev Alur, Osbert Bastani, Kishor Jothimurugan, Mateo Perez, Fabio Somenzi, and Ashutosh Trivedi.
\newblock Policy synthesis and reinforcement learning for discounted ltl.
\newblock In \emph{International Conference on Computer Aided Verification}, pages 415--435. Springer, 2023.

\bibitem[Black et~al.(2024)Black, Brown, Driess, Esmail, Equi, Finn, Fusai, Groom, Hausman, Ichter, et~al.]{black2024pi_0}
Kevin Black, Noah Brown, Danny Driess, Adnan Esmail, Michael Equi, Chelsea Finn, Niccolo Fusai, Lachy Groom, Karol Hausman, Brian Ichter, et~al.
\newblock $pi\_0 $: A vision-language-action flow model for general robot control.
\newblock \emph{arXiv preprint arXiv:2410.24164}, 2024.

\bibitem[Brody et~al.(2021)Brody, Alon, and Yahav]{brody2021attentive}
Shaked Brody, Uri Alon, and Eran Yahav.
\newblock How attentive are graph attention networks?
\newblock \emph{arXiv preprint arXiv:2105.14491}, 2021.

\bibitem[Brohan et~al.(2022)Brohan, Brown, Carbajal, Chebotar, Dabis, Finn, Gopalakrishnan, Hausman, Herzog, Hsu, et~al.]{brohan2022rt}
Anthony Brohan, Noah Brown, Justice Carbajal, Yevgen Chebotar, Joseph Dabis, Chelsea Finn, Keerthana Gopalakrishnan, Karol Hausman, Alex Herzog, Jasmine Hsu, et~al.
\newblock Rt-1: Robotics transformer for real-world control at scale.
\newblock \emph{arXiv preprint arXiv:2212.06817}, 2022.

\bibitem[Castro(2020)]{castro2020scalable}
Pablo~Samuel Castro.
\newblock Scalable methods for computing state similarity in deterministic markov decision processes.
\newblock In \emph{Proceedings of the AAAI Conference on Artificial Intelligence}, volume~34, pages 10069--10076, 2020.

\bibitem[Ferns et~al.(2004)Ferns, Panangaden, and Precup]{ferns2004metrics}
Norm Ferns, Prakash Panangaden, and Doina Precup.
\newblock Metrics for finite markov decision processes.
\newblock In \emph{UAI}, volume~4, pages 162--169, 2004.

\bibitem[Hopcroft(1971)]{hopcroft1971n}
John Hopcroft.
\newblock An n log n algorithm for minimizing states in a finite automaton.
\newblock In \emph{Theory of machines and computations}, pages 189--196. Elsevier, 1971.

\bibitem[Jothimurugan et~al.(2021)Jothimurugan, Bansal, Bastani, and Alur]{jothimurugan2021compositional}
Kishor Jothimurugan, Suguman Bansal, Osbert Bastani, and Rajeev Alur.
\newblock Compositional reinforcement learning from logical specifications.
\newblock \emph{Advances in Neural Information Processing Systems}, 34:\penalty0 10026--10039, 2021.

\bibitem[Liu et~al.(2022)Liu, Zhu, and Zhang]{liu2022goal}
Minghuan Liu, Menghui Zhu, and Weinan Zhang.
\newblock Goal-conditioned reinforcement learning: Problems and solutions.
\newblock \emph{arXiv preprint arXiv:2201.08299}, 2022.

\bibitem[Qiu et~al.(2023)Qiu, Mao, and Zhu]{qiu2023instructing}
Wenjie Qiu, Wensen Mao, and He~Zhu.
\newblock Instructing goal-conditioned reinforcement learning agents with temporal logic objectives.
\newblock \emph{Advances in Neural Information Processing Systems}, 36:\penalty0 39147--39175, 2023.

\bibitem[Ren et~al.(2025)Ren, Sundaresan, Sadigh, Choudhury, and Bohg]{ren2025motion}
Juntao Ren, Priya Sundaresan, Dorsa Sadigh, Sanjiban Choudhury, and Jeannette Bohg.
\newblock Motion tracks: A unified representation for human-robot transfer in few-shot imitation learning.
\newblock \emph{arXiv preprint arXiv:2501.06994}, 2025.

\bibitem[Rocamonde et~al.(2023)Rocamonde, Montesinos, Nava, Perez, and Lindner]{rocamonde2023vision}
Juan Rocamonde, Victoriano Montesinos, Elvis Nava, Ethan Perez, and David Lindner.
\newblock Vision-language models are zero-shot reward models for reinforcement learning.
\newblock \emph{arXiv preprint arXiv:2310.12921}, 2023.

\bibitem[Schaul et~al.(2015)Schaul, Horgan, Gregor, and Silver]{schaul2015universal}
Tom Schaul, Daniel Horgan, Karol Gregor, and David Silver.
\newblock Universal value function approximators.
\newblock In \emph{International conference on machine learning}, pages 1312--1320. PMLR, 2015.

\bibitem[Schulman et~al.(2017)Schulman, Wolski, Dhariwal, Radford, and Klimov]{schulman2017proximal}
John Schulman, Filip Wolski, Prafulla Dhariwal, Alec Radford, and Oleg Klimov.
\newblock Proximal policy optimization algorithms.
\newblock \emph{arXiv preprint arXiv:1707.06347}, 2017.

\bibitem[Sontakke et~al.(2023)Sontakke, Zhang, Arnold, Pertsch, B{\i}y{\i}k, Sadigh, Finn, and Itti]{sontakke2023roboclip}
Sumedh Sontakke, Jesse Zhang, S{\'e}b Arnold, Karl Pertsch, Erdem B{\i}y{\i}k, Dorsa Sadigh, Chelsea Finn, and Laurent Itti.
\newblock Roboclip: One demonstration is enough to learn robot policies.
\newblock \emph{Advances in Neural Information Processing Systems}, 36:\penalty0 55681--55693, 2023.

\bibitem[Strehl et~al.(2006)Strehl, Li, Wiewiora, Langford, and Littman]{strehl2006pac}
Alexander~L Strehl, Lihong Li, Eric Wiewiora, John Langford, and Michael~L Littman.
\newblock Pac model-free reinforcement learning.
\newblock In \emph{Proceedings of the 23rd international conference on Machine learning}, pages 881--888, 2006.

\bibitem[Vaezipoor et~al.(2021)Vaezipoor, Li, Icarte, and Mcilraith]{vaezipoor2021ltl2action}
Pashootan Vaezipoor, Andrew~C Li, Rodrigo A~Toro Icarte, and Sheila~A Mcilraith.
\newblock Ltl2action: Generalizing ltl instructions for multi-task rl.
\newblock In \emph{International Conference on Machine Learning}, pages 10497--10508. PMLR, 2021.

\bibitem[Yalcinkaya et~al.(2023)Yalcinkaya, Lauffer, Vazquez-Chanlatte, and Seshia]{yalcinkaya2023automata}
Beyazit Yalcinkaya, Niklas Lauffer, Marcell Vazquez-Chanlatte, and Sanjit Seshia.
\newblock Automata conditioned reinforcement learning with experience replay.
\newblock In \emph{NeurIPS 2023 Workshop on Goal-Conditioned Reinforcement Learning}, 2023.

\bibitem[Yalcinkaya et~al.(2024)Yalcinkaya, Lauffer, Vazquez-Chanlatte, and Seshia]{yalcinkaya24compositional}
Beyazit Yalcinkaya, Niklas Lauffer, Marcell Vazquez-Chanlatte, and Sanjit~A Seshia.
\newblock Compositional automata embeddings for goal-conditioned reinforcement learning.
\newblock In \emph{The Thirty-eighth Annual Conference on Neural Information Processing Systems}, 2024.

\bibitem[Yang et~al.(2021)Yang, Littman, and Carbin]{yang2021tractability}
Cambridge Yang, Michael Littman, and Michael Carbin.
\newblock On the (in) tractability of reinforcement learning for ltl objectives.
\newblock \emph{arXiv preprint arXiv:2111.12679}, 2021.

\bibitem[Zhang et~al.(2020)Zhang, McAllister, Calandra, Gal, and Levine]{zhang2020learning}
Amy Zhang, Rowan McAllister, Roberto Calandra, Yarin Gal, and Sergey Levine.
\newblock Learning invariant representations for reinforcement learning without reconstruction.
\newblock \emph{arXiv preprint arXiv:2006.10742}, 2020.

\end{thebibliography}

\newpage
\appendix

\section{Goal-Conditioned Reinforcement Learning}\label{sec:gcrl}

Here, for reference, we present the standard Goal-Conditioned Reinforcement Learning (GCRL) problem. We start by defining the environment model for conditioning on \emph{goals}, usually given as sets of states or continuous regions, as done by~\cite{schaul2015universal,liu2022goal}.

\begin{my_definition}[Goal-Conditioned MDP]
A goal-conditioned MDP extends the standard MDP by incorporating a goal space given by a goal distribution $\mathcal{G} \in \Delta(2^S)$, where $\mathcal{G}$ is a distribution over sets of states, and therefore a goal is a set of states.
It is defined as the tuple
\[
\mathcal{M}_\mathcal{G} = \langle S, A, P, R_\mathcal{G}, \iota_{\mathcal{G}}, \gamma \rangle,
\]
where:
\begin{itemize}
    \item $R_\mathcal{G}: S \times A \times G \to \mathbb{R}$ is the goal-conditioned reward function, and
    \item $\iota_\mathcal{G} : S \times G \to [0, 1]$ is the initial state-goal distribution defined by $\iota_\mathcal{G}(s, g) = \iota(s) \mathcal{G}(g)$.
\end{itemize}
\end{my_definition}

Given a goal-conditioned MDP, the standard GCRL problem is to find a policy that achieves a given goal, which was first introduced by~\cite{schaul2015universal}.

\begin{my_definition}[Goal-Conditioned Reinforcement Learning]\label{defn:gcrl}
Given a goal-conditioned MDP $\mathcal{M}_\mathcal{G}$, the \textbf{Goal-Conditioned Reinforcement Learning} (GCRL) problem is to find a policy
\[
\pi: S \times G \to \Delta(A),
\]
which maps a state-goal pair $(s,g)$ to a distribution over actions, that maximizes the expected cumulative discounted reward:
\[
J_{\mathcal{M}_\mathcal{G}}(\pi) = \mathbb{E}_{s_0, g \sim \iota_\mathcal{G}} \left[ \sum_{t=0}^{s_t \in g} \gamma^t R_\mathcal{G}(s_t, a_t, g) \right],
\]
where trace \(\{(s_t, a_t)\}_{t\ge0}\) is generated by $a_{t} \sim \pi(s_t)$ and $s_{t+1} \sim P(s_t,a_t)$ until $s_t \in g$ is reached.
The objective is to solve
\[
\pi^* = \arg\max_{\pi} J_{\mathcal{M}_\mathcal{G}}(\pi).
\]
\end{my_definition}

The standard GCRL formulation doesn't inherently allow for specifying temporally extended tasks since the goals are defined as sets of states.
In theory, one can extend the state definition to a product state and specify temporal tasks within that product state; however, such approaches limit the scalability of the GCRL framework.
On the other hand, our DFA-conditioned RL formulation given in \Cref{defn:dfa-rl} allows for specifying temporal tasks and enables optimal multi-task policy learning.

\section{Pseudometrics and Metrics}\label{sec:metric}

\begin{my_definition}[Pseudometric and Metric]
Let \(X\) be a nonempty set. A function \(d: X \times X \to [0,\infty)\) is called a \textbf{pseudometric} on \(X\) if for all \(x, y, z \in X\) the following conditions hold:
\begin{enumerate}
    \item \textbf{Non-negativity:} \(d(x,y) \ge 0\).
    \item \textbf{Identity on the diagonal:} \(d(x,x) = 0\).
    \item \textbf{Symmetry:} \(d(x,y) = d(y,x)\).
    \item \textbf{Triangle Inequality:} \(d(x,z) \le d(x,y) + d(y,z)\).
\end{enumerate}
If \(d(x,y) = 0\) implies \(x = y\), then \(d\) is a \textbf{metric}.
\end{my_definition}

Essentially, a metric is a function measuring the distance between any two points in a space, satisfying non-negativity, symmetry, the triangle inequality, and it equals zero if and only if the two points are identical. A pseudometric, on the other hand, allows distinct points to have a distance of zero, meaning it might not fully distinguish between different points in the space.

\section{Proofs of Theorems and Lemmas}

\pac*
\begin{proof}
    A DFA-conditioned MDP $\mathcal{M} \mid_L \mathcal{M}_{D_{\Sigma, n}}$ is defined over the state space \(S \times D_{\Sigma,n}\), where \(D_{\Sigma,n}\) is a DFA space. Since \(D_{\Sigma,n}\) is finite (as all DFAs in \(D_{\Sigma,n}\) has a finite alphabet $\Sigma$ and at most $n$ states), the product state space has size \(|S|\cdot|D_{\Sigma,n}|\) and is an MDP. Thus, any PAC-MDP algorithm that works for MDPs with state space size \(|S|\) will also work on the product MDP with state space size \(|S|\cdot|D_{\Sigma,n}|\), with sample complexity increasing by at most a factor polynomial in \(|D_{\Sigma,n}|\).
\end{proof}

\nobisim*
\begin{proof}
    If $\mathcal{A} \sim \mathcal{A}'$, then they must agree on acceptance, by~\Cref{defn:dfa-bisim}.
    We have \(R_{D_{\Sigma, n}}(\mathcal{A}) = 1\) if and only if \(\mathcal{A} = \mathcal{A}_\top\).
    Since \(\mathcal{A}\) and \(\mathcal{A}'\) are bisimilar, \(\mathcal{A} = \mathcal{A}_\top \iff \mathcal{A}' = \mathcal{A}_\top\).
    The same reasoning for the $-1$ reward case gives \(R_{D_{\Sigma, n}}(\mathcal{A}) = R_{D_{\Sigma, n}}(\mathcal{A}')\), i.e.,
    \(\mathcal{A}\) and \(\mathcal{A}'\) satisfy reward equivalence in \(\mathcal{M}_{D_{\Sigma, n}}\).
    For any \(a \in \Sigma\), \(T_{D_{\Sigma, n}}(\mathcal{A}, a)\) results in a DFA bisimilar to \(T_{D_{\Sigma, n}}(\mathcal{A}', a)\) due to \Cref{defn:dfa-bisim}. By induction on the structure of \(\mathcal{A}\) and \(\mathcal{A}'\), their transitions preserve bisimilarity, satisfying the transition equivalence.
    Therefore, we have \(\mathcal{A} \sim \mathcal{A}' \implies \mathcal{A} \sim_{\mathcal{M}_{D_{\Sigma, n}}} \mathcal{A}'\).

    If $\mathcal{A} \sim_{\mathcal{M}_{D_{\Sigma, n}}} \mathcal{A}'$, then \(R_{D_{\Sigma, n}}(\mathcal{A}) = R_{D_{\Sigma, n}}(\mathcal{A}')\).
    Thus, $\mathcal{A}$ and $\mathcal{A}'$ must agree on acceptance by \Cref{defn:dfa-mdp}.
    For every \(a \in \Sigma\), \(T_{D_{\Sigma, n}}(\mathcal{A}, a) \sim_{\mathcal{M}_{D_{\Sigma, n}}} T_{D_{\Sigma, n}}(\mathcal{A}', a)\) by \Cref{defn:dfa-mdp}. By induction on the DFA transition structure (which is finite), \(T_{D_{\Sigma, n}}(\mathcal{A}, a)\) and \(T_{D_{\Sigma, n}}(\mathcal{A}', a)\) are bisimilar. As transitions under all \(a \in \Sigma\) preserve bisimilarity, the initial states \(q_0\) and \(q_0'\) must be related under the bisimulation relation. Thus, \(\mathcal{A}\) and \(\mathcal{A}'\) are bisimilar, i.e., \(\mathcal{A} \sim \mathcal{A}' \impliedby \mathcal{A} \sim_{\mathcal{M}_{D_{\Sigma, n}}} \mathcal{A}'\).
\end{proof}

\operator*
\begin{proof}
    Let $\mathcal{M} = \langle S, A, P, R, \iota, \gamma \rangle$ be an MDP. Define
    \[
        d^{k}(s, t)
        \gets
        \arg\max_{a\in A}
        \left\{
        \left|
        R(s, a) - R(t, a)
        \right|
        +
        \gamma
        \mathcal{W}_1(d)
        \left(
        P(s, a), P(t, a)
        \right)
        \right\},
    \]
    where $\mathcal{W}_1$ is the 1-Wasserstein metric.
    \cite{ferns2004metrics} showed that this operator has a unique fixed point $d^{*}$, and $d^{*}$ is a bisimulation metric.
    Later, \cite{castro2020scalable} proved that if $\mathcal{M}$ is deterministic, with transition function $T$, then the 1-Wasserstein metric above implies as follows:
    \[
        \mathcal{W}_1(d)
        \left(
        P(s, a), P(s, a)
        \right)
        =
        d
        \left(
        T(s, a), T(s, a)
        \right).
    \]
    We write the distance and the policy update separately since we want to learn the $\arg\max$--hard to compute directly.
    However, then we need to prove that the distance metric still has a unique fixed point when updated with actions from a policy being simultaneously learned.
    A useful result due to \cite{zhang2020learning} (which we present by combining with the result of \cite{castro2020scalable} and our notation) shows that given a continuously improving policy $\pi$, the following operator:
    \[
        d^{k}(s, t)
        \gets
        \left|
        R(s, \pi(s, t)) - R(t, \pi(s, t))
        \right|
        +
        \gamma
        d
        \left(
        T(s, \pi(s, t)), T(s, \pi(s, t))
        \right)
    \]
    has a unique fixed point $d^*$, and $d^*$ is a $\pi^*$-bisimulation metric.
    In our case, $\pi^*$ is the $\arg\max$ policy and therefore the unique fixed point $d^*$ is a bisimulation metric.

\end{proof}

\unique*
\begin{proof}
By \Cref{thm:operator}, $d^*(\mathcal{A}, \mathcal{A}') = \| \hat{\phi}^*(\mathcal{A}) - \hat{\phi}^*(\mathcal{A}') \|_2$ is a bisimulation metric.
Therefore, $d^*(\mathcal{A}, \mathcal{A}') = 0 \implies \mathcal{A} \sim_{\mathcal{M}_{D_{\Sigma, n}}} \mathcal{A}'$ by \Cref{defn:bisim-metric} and thus, by \Cref{lemma:no-bisim}, we have $\mathcal{A} \sim \mathcal{A}'$.
As $d^*(\mathcal{A}, \mathcal{A}') = 0$ implies $\phi^*(\mathcal{A}) = \phi^*(\mathcal{A}')$,
we have $\mathcal{A} \sim \mathcal{A}' \iff \phi^*(\mathcal{A}) = \phi^*(\mathcal{A}')$.
The forward direction is true since if $\mathcal{A} \sim \mathcal{A}'$, then $d^*(\mathcal{A}, \mathcal{A}') = 0$ by \Cref{defn:bisim-metric}; thus, $\phi^*(\mathcal{A}) = \phi^*(\mathcal{A}')$.
\end{proof}

\main*
\begin{proof}
    The optimal encoder $\phi^*$ maps two DFAs to the same latent representation if and only if they are bisimilar by \Cref{thm:unique}. So, every unique task in $D_{\Sigma, n}$ is represented in $\mathcal{Z}$.
    Therefore, the problem given in \Cref{defn:dfa-rl} can be equivalently reformulated over \( \mathcal{Z} \), which solves \Cref{problem:main}.
\end{proof}

\end{document}